%% file: fast4.tex
\documentclass{amsart} 
\usepackage{amsmath,amssymb,natbib,amsthm,xfrac,subfigure}
\usepackage{url} 
\usepackage{tabularx,color} 
\usepackage{dsfont} 
\usepackage{epsf} 
\usepackage{algorithmic,algorithm,mathtools} 
\usepackage{wrapfig}

\usepackage{graphicx}

\newtheorem{definition}{Definition}
\newtheorem{theorem}{Theorem}

\newtheorem{lemma}{Lemma}

\newtheorem{property}{Property}
\newtheorem{example}{Example}
\newtheorem{remark}{Remark}

\newcommand{\ERED}[1]{\boldsymbol{\textcolor{red}{#1}}} 
\title[Sinkhorn Distances]{Sinkhorn Distances: Lightspeed Computation of Optimal Transportation Distances}

\input defs.tex

\input marcodefs.tex

\providecommand{\norm}[1]{\lVert#1\rVert}

\def\kl{\mathbf{KL}}
\def\KL{\mathbf{KL}}

\author{Marco Cuturi}
\address{Graduate School of Informatics, Kyoto University}
\email{mcuturi@i.kyoto-u.ac.jp} 

\begin{document}

\maketitle

\begin{abstract}
Optimal transportation distances are a fundamental family of parameterized distances for histograms. Despite their appealing theoretical properties, excellent performance in retrieval tasks and intuitive formulation, their computation involves the resolution of a linear program whose cost is prohibitive whenever the histograms' dimension exceeds a few hundreds. We propose in this work a new family of optimal transportation distances that look at transportation problems from a maximum-entropy perspective. We smooth the classical optimal transportation problem with an entropic regularization term, and show that the resulting optimum is also a distance which can be computed through Sinkhorn-Knopp's matrix scaling algorithm at a speed that is several orders of magnitude faster than that of transportation solvers. We also report improved performance over classical optimal transportation distances on the MNIST benchmark problem.
\end{abstract}

\section{Introduction}
Optimal transportation distances~\citep[\S6]{villani09} -- also known as Earth Mover's following the seminal work of~\citet{rubner1997earth} and their application to computer vision -- hold a special place among other distances in the probability simplex. Compared to other classic distances or divergences, such as Hellinger, $\chi_2$, Kullback-Leibler or Total Variation, they are the only ones to be parameterized. This parameter -- the \emph{ground metric} -- plays an important role to handle high-dimensional histograms: the ground metric provides a natural way to handle \emph{redundant} features that are bound to appear in high-dimensional histograms (think synonyms for bags-of-words), in the same way that Mahalanobis distances can correct for statistical \emph{correlations} between vector coordinates.

The central role played by histograms and bags-of-features in most data analysis tasks and the good performance of optimal transportation distances in practice has generated ample interest, both from a theoretical point of view ~\citep{levina2001earth,indyk,naor-2005,indyk2009} and a pracical aspect, mostly to compare images~\citep{GraumanD04,ling2007efficient,gudmundsson2007small,shirdhonkar2008approximate}. Optimal transportation distances have, however, a very clear drawback. No matter what the algorithm employed -- network simplex or interior point methods -- their cost scales at least in $O(d^3log(d))$ when computing the distance between a pair of histograms of dimension $d$, in the general case where no restrictions are placed upon the ground metric parameter~\citep[\S2.1]{Pele-iccv2009}. This speed can be improved by ensuring that the ground metric observes certain constraints and/or by accepting some approximation errors. However, when these restrictions do not apply, computing a single distance between a pair of histograms of dimension in the few hundreds can take more than a few seconds. This issue severely hinders the applicability of optimal transportation distances in large-scale data analysis and goes as far as putting into question their relevance within the field of machine learning.

Our aim in this paper is to show that the optimal transportation problem can be regularized by an entropic term, following the maximum-entropy principle. We argue that this regularization is intuitive given the geometry of the optimal transportation problem and has, in fact, been long known and favored in transportation theory~\citep{erlander1990gravity}. From an optimization point of view, this regularization has multiple virtues, among which that of turning this LP into a strictly convex problem that can be solved extremely quickly with the Sinkhorn-Knopp matrix scaling algorithm  \citep{sinkhorn1967concerning,knight2008sinkhorn}. This algorithm exhibits  linear convergence and can be trivially parallelized -- \emph{it can be vectorized}. It is therefore amenable to large scale executions on parallel platforms such as GPGPUs. From a practical perspective, we show that, on the benchmark task of classifying MNIST digits, Sinkhorn distances perform better than the EMD and can be computed several orders of magnitude faster over a large sample of dimensions \emph{without making any assumption on the ground metric}. We believe this paper contains all the ingredients that are required for optimal transportation distances to be at last applied on high-dimensional datasets and attract again the attention of the machine learning community.

This paper is organized as follows: we provide reminders on optimal transportation theory in Section~\ref{sec:rem}, introduce Sinkhorn distances in Section~\ref{sec:grav} and provide algorithmic details in Section~\ref{sec:computing}. We follow with an empirical study in Section~\ref{sec:exp} before concluding.

\section{Reminders on Optimal Transportation}\label{sec:rem}
\subsection{Transportation Tables and Joint Probabilities}  In what follows, $\langle \cdot,\cdot \rangle$ stands for the Frobenius dot-product. For two histograms $r$ and $c$ in the simplex $\Sigma_d\defeq \{x\in \RR^d_+: x^T\ones_d=1\}$,  we write $U(r,c)$ for the transportation polytope of $r$ and $c$, namely the polyhedral set of $d\times d$ matrices:
$$U(r,c)\defeq \{P\in\RR_+^{d\times d}\; |\; P\ones_d=r, P^T\ones_d=c\},$$
where $\ones_d$ is the $d$ dimensional vector of ones. $U(r,c)$ contains all nonnegative $d\times d$ matrices with row and column sums $r$ and $c$ respectively. $U(r,c)$ has a probabilistic interpretation: for $X$ and $Y$ two multinomial random variables taking values in $\{1,\cdots,d\}$, each with distribution $r$ and $c$ respectively, the set $U(r,c)$ contains all possible \emph{joint probabilities} of $(X,Y)$. Indeed, any matrix $P\in U(r,c)$ can be identified with a joint probability for $(X,Y)$ such that $p(X=i,Y=j)=p_{ij}$. Such joint probabilities are also known as \emph{contingency tables}. We define the entropy $h$ and the Kullback-Leibler divergences of these tables and their marginals as
\begin{gather*}
r\in\Sigma_d,\quad h(r)=-\sum_{i=1}^d r_i \log r_i,\quad \quad P\in U(r,c),\quad h(P)=-\sum_{i,j=1}^d p_{ij} \log p_{ij}\\ P,Q \in U(r,c),\quad \kl(P\|Q)=\sum_{ij}p_{ij}\log\frac{p_{ij}}{q_{ij}}.
\end{gather*}

\subsection{Optimal Transportation} Given a $d\times d$ cost matrix $M$, the cost of mapping $r$ to $c$ using a transportation matrix (or joint probability) $P$ can be quantified  as $\dotprod{P}{M}$. The following problem:
$$
d_M(r,c) \defeq \min_{P\in U(r,c)} \dotprod{P}{M}.
$$
is called an \emph{optimal transportation} problem between $r$ and $c$ given cost $M$. An optimal table $P^\star$ for this problem can be obtained with the network simplex~\citep[\S9]{ahuja1993network} as well as other approaches~\citep{orlin1993faster}. The optimum of this problem, $d_M(r,c)$, is a distance~\citep[\S6.1]{villani09} whenever the matrix $M$ is itself a metric matrix, namely whenever $M$ belongs to the cone of distance matrices~\citep{avis1980extreme,brickell2008metric}:
$$\Mcal=\{M\in\RR^{d\times d}_+:\forall i\leq d, m_{ii}=0;\, \forall i,j,k\leq d, m_{ij}\leq m_{ik} + m_{kj}\}.$$
For a general matrix $M$, the worst case complexity of computing that optimum with any of the algorithms known so far scales in $O(d^3\log d)$ and turns out to be super-cubic in practice as well~\citep[\S2.1]{Pele-iccv2009}. Much faster speeds can be obtained however when placing all sorts of restrictions on $M$ and accepting approximated solutions, albeit at a cost in performance~\citep{GraumanD04} and a loss in applicability.

\section{Sinkhorn Distances}\label{sec:grav}
We consider in this section a family of optimal transportation distances whose feasible set is the not the whole of $U(r,c)$, but a parameterized restricted set of joint probability matrices.

\subsection{Entropic Constraints on Joint Probabilities}
We recall a basic information theoretic inequality~\citep[\S2]{cover91elements} which applies to all joint probabilities: 
\begin{equation}\label{eq:basic}
	\forall r,c \in \Sigma_d, \forall P\in U(r,c), h(P) \leq h(r)+h(c).
\end{equation}

This bound is tight, since the table $rc^T$ -- known as the independence table~\citep{good1963maximum} -- has an entropy of $h(rc^T)=h(r)+h(c)$. By the concavity of entropy, we can introduce the convex set $U_\alpha(r,c)\subset U(r,c)$ as
$$U_\alpha(r,c)\defeq \{P\in U(r,c)\,|\,\KL(P\|rc^T)\leq \alpha \} = \{P\in U(r,c)\,|\, h(P)\geq h(r)+h(c)-\alpha\}$$

These definitions are indeed equivalent, since one can easily check that 
$$\KL(P\|rc^T)= h(r)+h(c)-h(P),$$ a quantity which is also the mutual information $I(X\|Y)$ of two random variables $(X,Y)$ should they follow the joint probability $P$~\cite[\S2]{cover91elements}. Hence, all tables $P$ whose Kullback-Leibler divergence to the table $rc^T$ is constrained to lie below a certain threshold can be interpreted as the set of tables $P$ in $U(r,c)$ which have \emph{sufficient} entropy with respect to $h(r)$ and $h(c)$, or joint probabilities which display a small enough \emph{mutual information}.

As a classic result of linear optimization, the optimum of classical optimal transportation distances is achieved on vertices of $U(r,c)$, that is $d\times d$ matrices with only up to $2d-1$ non-zero elements~\citep[\S8.1.3]{brualdi2006combinatorial}. Such plans can be interpreted as quasi-deterministic joint probabilities, since if $p_{ij}>0$, then very few values $p_{ij'}$ will have a non-zero probability. By mitigating the transportation cost objective with an entropic constraint, which is equivalent to following the max-entropy principle~\citep{PhysRev.106.620,dudik2006maximum} and thus for a given level of the cost look for the most smooth joint probability, we argue that we can provide a more robust notion of distance between histograms. Indeed, for a given pair $(r,c)$, finding plausible transportation plans with low cost (where plausibility is measured by entropy) is more informative than finding \emph{extreme plans} that are extremely unlikely to appear in nature.

We note that the idea of regularizing the transportation problem was also considered recently by~\citet{ferradans2013regularized}. In their work, ~\citeauthor{ferradans2013regularized} also argue that an optimal matching may not be sufficiently regular in vision applications (color transfer), and that these undesirable properties can be handled through an adequate relaxation and penalization (through graph-based norms) of the transportation problem. While~\citet{ferradans2013regularized} penalize the transportation problem to obtain a more regular transportation plan, we believe that an entropic regularization yields here a better distance. An illustration of this idea is provided in Figure~\ref{fig:jolidessin}.  For reasons that will become clear in Section~\ref{sec:computing}, we call such distances \emph{Sinkhorn distances}.
\begin{definition}[Sinkhorn Distances]\label{def:sinkhorn} $\displaystyle d_{M,\alpha}(r,c)\defeq\min_{P\in U_\alpha(r,c)}\dotprod{P}{M}$
\end{definition}

\begin{figure}\centering
{\LARGE{\BC\scalebox{1}{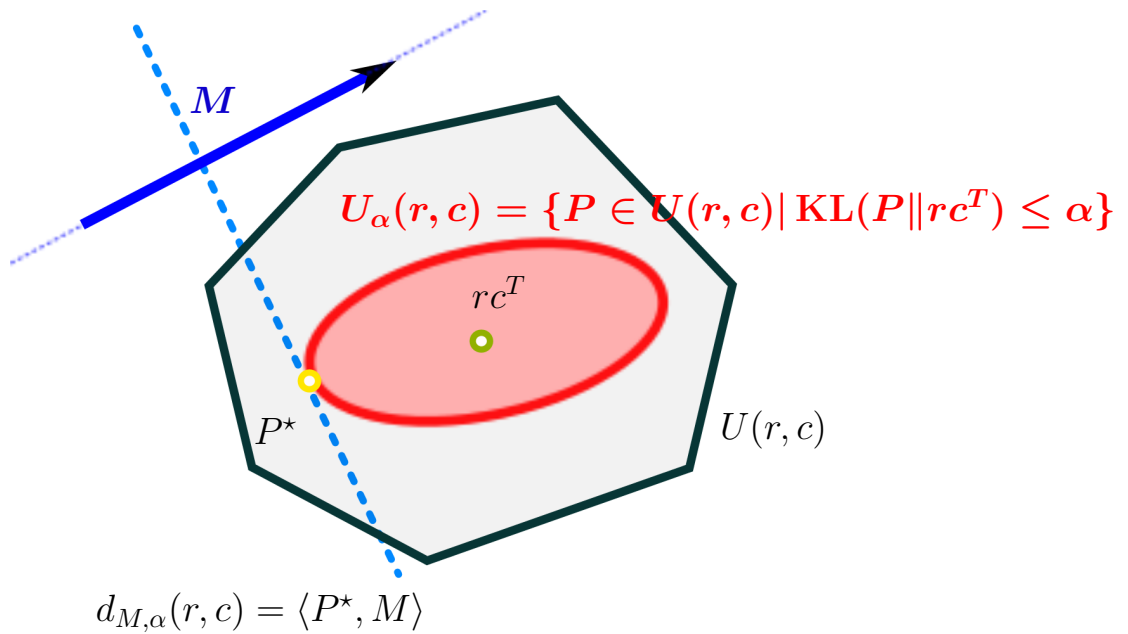}\EC}}
\caption{Schematic view of the transportation polytope and the Kullback-Leibler ball of level $\alpha$ that surrounds the independence table $rc^T$. The Sinkhorn distance is the dot product of $M$ with the optimal transportation table in that ball.}\label{fig:jolidessin}
\end{figure}

\subsection{Metric Properties}
When $\alpha$ is large enough, the Sinkhorn distance coincides with the classic optimal transportation distance. When $\alpha=0$, the Sinkhorn distance has a closed form and becomes a negative definite kernel if one assumes that $M$ is itself a negative definite distance, that is a Euclidean distance matrix.

\begin{property}\label{prop:nsd} For $\alpha$ large enough, the Sinkhorn distance $d_{M,\alpha}$ is the transportation distance $d_{M}$.\end{property}
\begin{proof}Since for any $P\in U(r,c), h(P)$ is lower bounded by $\tfrac{1}{2}(h(r)+h(c))$, we have that for $t$ large enough $U_{t}(r,c)= U(r,c)$ and thus both quantities coincide.\end{proof}

\begin{property}[Independence Kernel] When $\alpha=0$ and $M$ is a Euclidean Distance Matrix\footnote{$\exists n, \exists \varphi_1,\cdots,\varphi_d\in\RR^n$ such that $m_{ij}=\norm{\varphi_i - \varphi_j}_2^2$ \citep[\S5]{dattorro2005convex}. Recall that, in that case, $M.^t=[m_{ij}^t]$, $0<t<1$ is also a Euclidean distance matrix~\cite[p.78,\S3.2.10]{berg84harmonic} }, the Sinkhorn distance has the explicit form $d_{M,0}=r^T M c$. $d_{M,0}$ is a negative definite kernel, \ie\, $e^{-t r^TMc}$ is a positive definite kernel $\forall t>0$. We call this kernel the independence kernel.\end{property}

The proof is provided in the appendix. Beyond these two extreme cases, the main theorem of this section states that Sinkhorn distances are symmetric and satisfy triangle inequalities for all possible values of $\alpha$. Since for $\alpha$ small enough $d_{M,\alpha}(r,r)>0$ for any $r$ such that $h(r)>0$, Sinkhorn distances cannot satisfy the \emph{coincidence axiom}\footnote{satisfied if $d(x,y)=0\Leftrightarrow x=y$ holds for all $x,y$}. However, multiplying $d_{M,\alpha}$ by $\ones_{r\ne c}$ suffices to recover the coincidence property if needed.

\begin{theorem}\label{theo:metric} For all $\alpha\geq 0$ and $M\in\Mcal$, $d_{M,\alpha}$ is symmetric and satisfies all triangle inequalities. The function $(r,c)\mapsto \ones_{r\ne c} d_{M,\alpha}(r,c)$ satisfies all three distance axioms. 
\end{theorem}
The gluing lemma~\cite[Lemma 7.6]{villani} plays a crucial role to prove that optimal transportation distances are indeed distances. The version we use below is slightly different since it incorporates the entropic constraint.
\begin{lemma}[Gluing Lemma With Entropic Constraint]\label{lem:gluing} Let $\alpha\geq 0$ and $x,y,z$ be three elements of $\Sigma_d$. Let $P\in U_\alpha(x,y)$ and $Q\in U_\alpha(y,z)$ be two joint probabilities in the transportation polytopes of $(x,y)$ and $(y,z)$ with sufficient entropy. Let $S$ be the $d\times d$ matrix whose $(i,k)$'s coefficient is $s_{ik}\defeq \sum_{j} \frac{p_{ij}q_{jk}}{y_j}$. Then $S\in U_{\alpha}(x,z)$.
\end{lemma}
The proof is provided in the appendix. We can prove the triangle inequality for $d_{M,\alpha}$ by using the same proof strategy than that used for classical transportation distances.

\emph{Proof of Theorem~\ref{theo:metric}.} The symmetry of $d_{M,\alpha}$ is a direct result of $M$'s symmetry. Let $x,y,z$ be three elements in $\Sigma_d$. Let $P\in U_{\alpha}(x,y)$ and $Q\in U_{\alpha}(y,z)$ be the optimal solutions obtained when computing $d_{M,\alpha}(x,y)$ and $d_{M,\alpha}(y,z)$ respectively. Using the matrix $S$ of $U_{\alpha}(x,z)$ provided in Lemma~\ref{lem:gluing}, we proceed with the following chain of inequalities:
	$$\begin{aligned}
	d_{M,\alpha}(x,z)&=\min_{P\in U_\alpha(x,z)}\dotprod{X}{M} \leq \dotprod{S}{M} = \sum_{ik} m_{ik}\sum_{j} \frac{p_{ij}q_{jk}}{y_j} \\ & \leq \sum_{ijk} \left(m_{ij}+m_{jk}\right) \frac{p_{ij}q_{jk}}{y_j} = \sum_{ijk} m_{ij} \frac{p_{ij}q_{jk}}{y_j} +m_{jk} \frac{p_{ij}q_{jk}}{y_j}  \\
	& =\sum_{ij} m_{ij}p_{ij} \sum_k \frac{q_{jk}}{y_j} + \sum_{jk} m_{jk} q_{jk} \sum_i \frac{p_{ij}}{y_j}\\
	&= \sum_{ij} m_{ij}p_{ij} + \sum_{jk} m_{jk} q_{jk} = d_{M,\alpha}(x,y)+d_{M,\alpha}(y,z).\,\blacksquare\end{aligned}
	$$

\section{Computing Sinkhorn Distances with the Sinkhorn-Knopp Algorithm}\label{sec:computing}
Recall that the Sinkhorn distance (Definition ~\ref{def:sinkhorn}) is  defined through a hard constraint on the entropy of $h(P)$ relative to $h(r)$ and $h(c)$.
In what follows, we consider the same program with a Lagrange multiplier for the entropy constraint,
\begin{equation}\label{eq:pen}
\boxed{d_{M}^{\lambda}(r,c) \defeq \dotprod{P^\lambda}{M},\,\text{ where } P^\lambda=\argmin_{P\in U(r,c)} \dotprod{P}{M} - \frac{1}{\lambda} h(P).}
\end{equation}
By duality theory we have that for every pair $(r,c)$, to each $\alpha$ corresponds an $\lambda\in[0,\infty]$ such that $d_{M,\alpha(r,c)}=d_{M}^{\lambda}(r,c)$. We call $d_{M}^\lambda$ the dual-Sinkhorn divergence and show that it can be computed at a much cheaper cost than the classical optimal transportation problem for reasonable values of $\lambda$.

\subsection{Computing $d_{M}^\lambda$} When $\lambda>0$, the solution $P^\lambda$ is unique by strict convexity of minus the entropy. In fact, $P^\lambda$ is necessarily of the form $u_i e^{-\lambda m_{ij}} v_j$, where $u$ and $v$ are two non-negative vectors uniquely defined up to a multiplicative factor. 
\begin{algorithm}[H]
        \begin{algorithmic}
	\caption{Computation of $d^\lambda_M(r,c)$ using Sinkhorn-Knopp's fixed point iteration}\label{alg:sk}
          \STATE \textbf{Input} \texttt{M}, $\lambda$, \texttt{r}, \texttt{c}.
		  \STATE \texttt{I=(r>0); r=r(I); M=M(I,:); \texttt{K=exp(-$\lambda$*M)}}
		  \STATE Set \texttt{x}=\texttt{ones(length(r),size(c,2))/length(r)};
		  \WHILE{\texttt{x} changes}
		  \STATE \texttt{x=diag(1./r)*K*(c.*(1./(K'*(1./x))))}
		  \ENDWHILE
		  \STATE  \texttt{u=1./x;  v=c.*(1./(K'*u))}
		  \STATE $d_{M}^\lambda$(\texttt{r},\texttt{c})=\texttt{sum(u.*((K.*M)*v))}
        \end{algorithmic}
      \end{algorithm}    	
This well known fact in transportation theory~\citep{erlander1990gravity} can be indeed checked by forming the Lagrangian $\Lcal(P,\alpha,\beta)$ of the objective of Equation~\eqref{eq:pen} using $\alpha,\beta\geq \mathbf{0}_d$ for each of the two equality constraints in $U(r,c)$. For these two cost vectors $\alpha,\beta$,
$$\Lcal(P,\alpha,\beta)=\sum_{ij} \frac{1}{\lambda}p_{ij}\log p_{ij}+ p_{ij}m_{ij} + \alpha^T(P \ones_d-r) + \beta^T(P^T\ones_d-c)$$
We obtain then, for any couple $(i,j)$, that if $\frac{\partial \Lcal}{\partial p_{ij}^\lambda}=0$, then 
$$
p_{ij}^\lambda= e^{-\frac{1}{2}-\lambda \alpha_i} e^{-\lambda m_{ij}}e^{-\frac{1}{2}-\lambda \beta_j},
$$
and thus recover the form provided above. $P^\lambda$ is thus, by~\citeauthor{sinkhorn1967concerning}'s theorem~\citeyearpar{sinkhorn1967concerning}, the \emph{only matrix} with row-sum $r$ and column-sum $c$ of the form
\begin{equation}\label{eq:sinkhorn}\exists u,v>\mathbf{0}_d : P^\lambda = \diag(u) e^{-\lambda M} \diag(v).\end{equation}
Given $e^{-\lambda M}$ and marginals $r$ and $c$, it is thus sufficient to run enough iterations of~\citeauthor{sinkhorn1967concerning}'s algorithm to converge to a solution $P^\lambda$ of that problem. We provide a one line implementation in Algorithm~\ref{alg:sk}. The case where some coordinates of $r$ or $c$ are null can be easily handled by selecting those elements of $r$ that are strictly positive to obtain the desired table, as shown in the first line of Algorithm~\ref{alg:sk}. Note that Algorithm~\ref{alg:sk} is \emph{vectorized}: it can be used as such to compute the distance between $r$ and a \emph{family of histograms} $C=[c_1,\cdots,c_N]$ by replacing $c$ with $C$. These $O(d^2N)$ linear algebra operations can be very quickly executed by using a GPGPU. 

\subsection{Computing $d_{M,\alpha}$ through $d_{M}^\lambda$} With a naive approach, $d_{M,\alpha}$ can be obtained by computing $d_{M}^\lambda$ iteratively until the entropy of the solution $P^\lambda$ has reached an adequate value $h(r)+h(c)-\alpha$. Since the entropy of $P^\lambda$ decreases monotonically when $\lambda$ increases, this search can be carried out by simple bisection, starting with a small $\lambda$ which is iteratively increased. In what follows, we only consider the dual-Sinkhorn divergence $d_{M}^\lambda$ since it is cheaper to compute and displays good performances in itself. We believe that more clever approaches can be applied to calculate exactly $d_{M,\alpha}$, and we leave this for future work. In the rest of this paper we will now refer to $d_{M}^\lambda$ as the Sinkhorn distance, despite the fact that it is not provably a distance.
\section{Experimental Results}\label{sec:exp}

\subsection{MNIST Digits}\label{subsec:mnist}
We test the performance of Sinkhorn distances on the MNIST digits\footnote{\texttt{http://yann.lecun.com/exdb/mnist/}} dataset, on which the ground metric has a natural interpretation in terms of pixel distances. Each digit is provided as a vector of intensities on a $20\times 20$ pixel grid. We convert each image into a histogram by normalizing each pixel intensity by the total sum of all intensities . 
We consider a subset of $N$ points in the training set of the database, where $N$ ranges within $\{3,5,12,17,25\}\times 10^3$ datapoints.

\begin{figure}\centering
\fbox{\includegraphics[width=9cm]{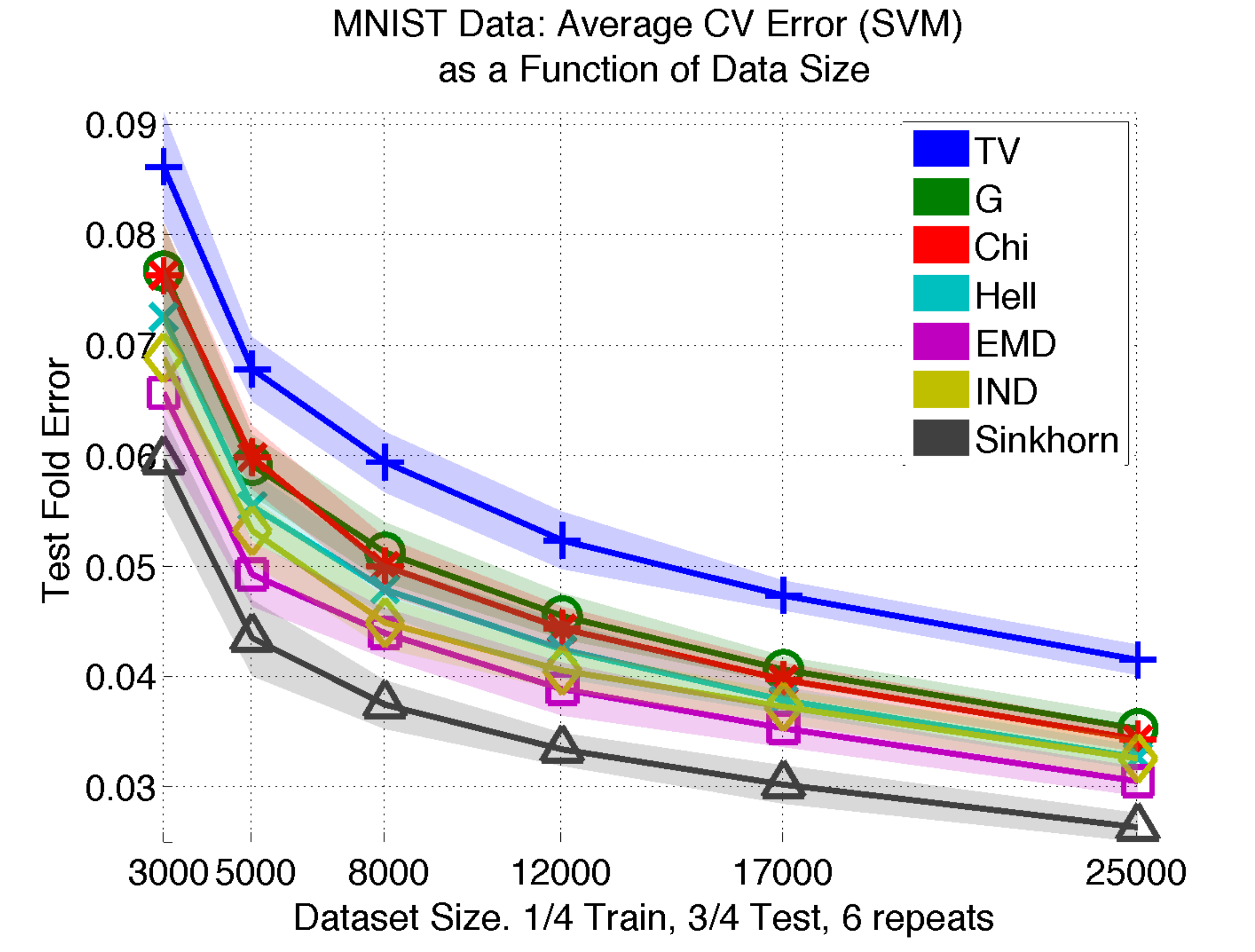}}
\caption{Average test errors with shaded confidence intervals. Errors are computed using 1/4 of the dataset for train and 3/4 for test. Errors are averaged over 4 folds $\times$ 6 repeats = 24 experiments.}
\end{figure}
\subsubsection{Experimental setting}  For each subset of size $N$, we provide mean and standard deviation of classification error using a 4 fold (3 test, 1 train) cross validation scheme repeated 6 times, resulting in 24 different experiments. We study the performance of different distances with the following parameter selection scheme: for each distance $d$, we consider the kernel $e^{-d/t}$, where $t>0$ is chosen by cross validation individually for each training fold within the set $\{1,q_{10}(d),q_{20}(d),q_{50}(d)\}$, where $q_s$ is the $s\%$ quantile of a subset of distances observed in the training fold. 
We regularize non-positive definite kernel matrices resulting from this computation by adding a sufficiently large diagonal term. SVM's were run with \texttt{libsvm} (one-vs-one) for multiclass classification, the regularization constant $C$ being selected by 2 folds/2 repeats cross-validation on the training fold in the set $10^{-2:2:4}$

\subsubsection{Distances} The Hellinger, $\chi_2$, Total Variation and squared Euclidean (Gaussian kernel) distances are used as such. We set the ground metric $M$ to be the Euclidean distance between the $20\times 20$ points in the grid, resulting in a $400\times 400$ distance matrix. We also tried to use Mahalanobis distances on this example with a positive definite matrix equal to \texttt{exp(-tM.\^{}2)}, \texttt{t>0},  as well as its inverse, with varying values of $t$ but none of the results proved competitive. For the Independence kernel, since any Euclidean distance matrix is valid, we consider $[m_{ij}^a]$ where $a\in\{0.01,0.1,1\}$ and choose $a$ by cross-validation on the training set. Smaller values of $a$ seem to be preferable. We select the entropic penalty $\lambda$ of Sinkhorn distances so that the matrix $e^{-\lambda M}$ is relatively diagonally dominant and the resulting transportation not too far from the classic optimal transportation. We select $\lambda$ for each training fold by internal cross-validation within $\{5,7,9,11\}\times 1/q_{50}(M)$ where $q_{50}(M)$ is the median distance between pixels on the grid. We set the number of fixed-point iterations to an arbitrary number of 20 iterations. In most (though not all) folds, the value $\lambda=9$ comes up as the best setting. The Sinkhorn distance beats by a safe margin all other distances, including the EMD.

\begin{figure}\centering
\fbox{\includegraphics[width=9cm]{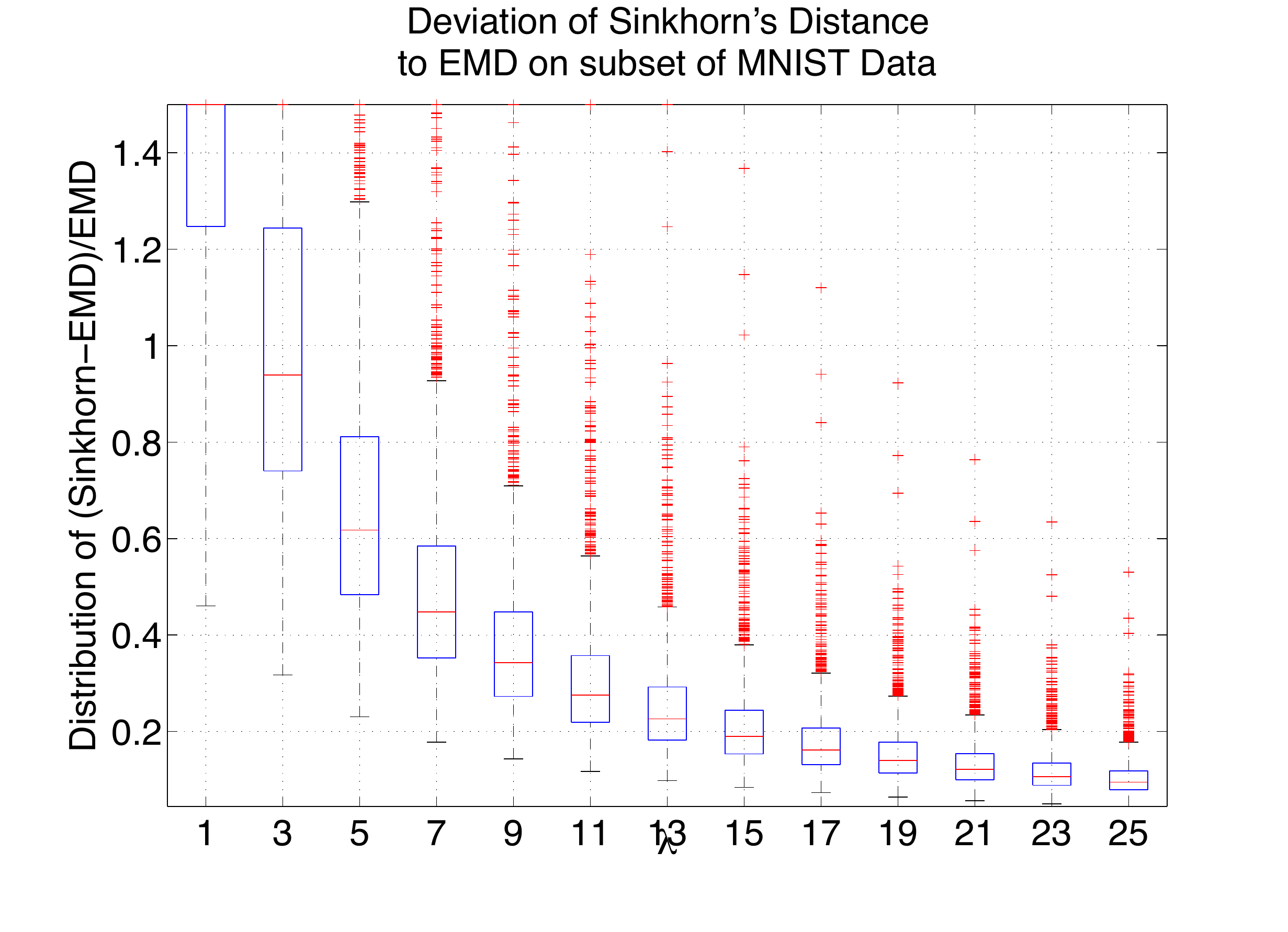}}
\caption{Decrease of the gap between the Sinkhorn distance and the EMD on the MNIST dataset.}\label{fig:decrease}
\end{figure}

\subsection{Does the Sinkhorn Distance Converge to the EMD?}

We study in this section the convergence of Sinkhorn distances towards classical optimal transportation distances as $\lambda$ gets bigger. Because of the additional penalty that appears in \eqref{eq:pen} program, $d_{M}^\lambda(r,c)$ is necessarily larger than $d_M(r,c)$, and we expect this gap to decrease as $\lambda$ increases. Figure~\ref{fig:decrease} illustrates this by plotting the boxplot of distributions of $(d_{M}^\lambda(r,c)-d_M(r,c))/d_M(r,c)$ over $40^2$ pairs of distinct points taken in the MNIST database. As can be observed, even with large values of $\lambda$, Sinkhorn distances hover above the values of EMD distances by about $10\%$. For practical values of $\lambda$ such as $\lambda=9$ selected above we do not expect the Sinkhorn distance to be numerically close to the EMD, nor believe it to be a desirable property.

\begin{figure}\centering
\centering\fbox{\includegraphics[width=9cm]{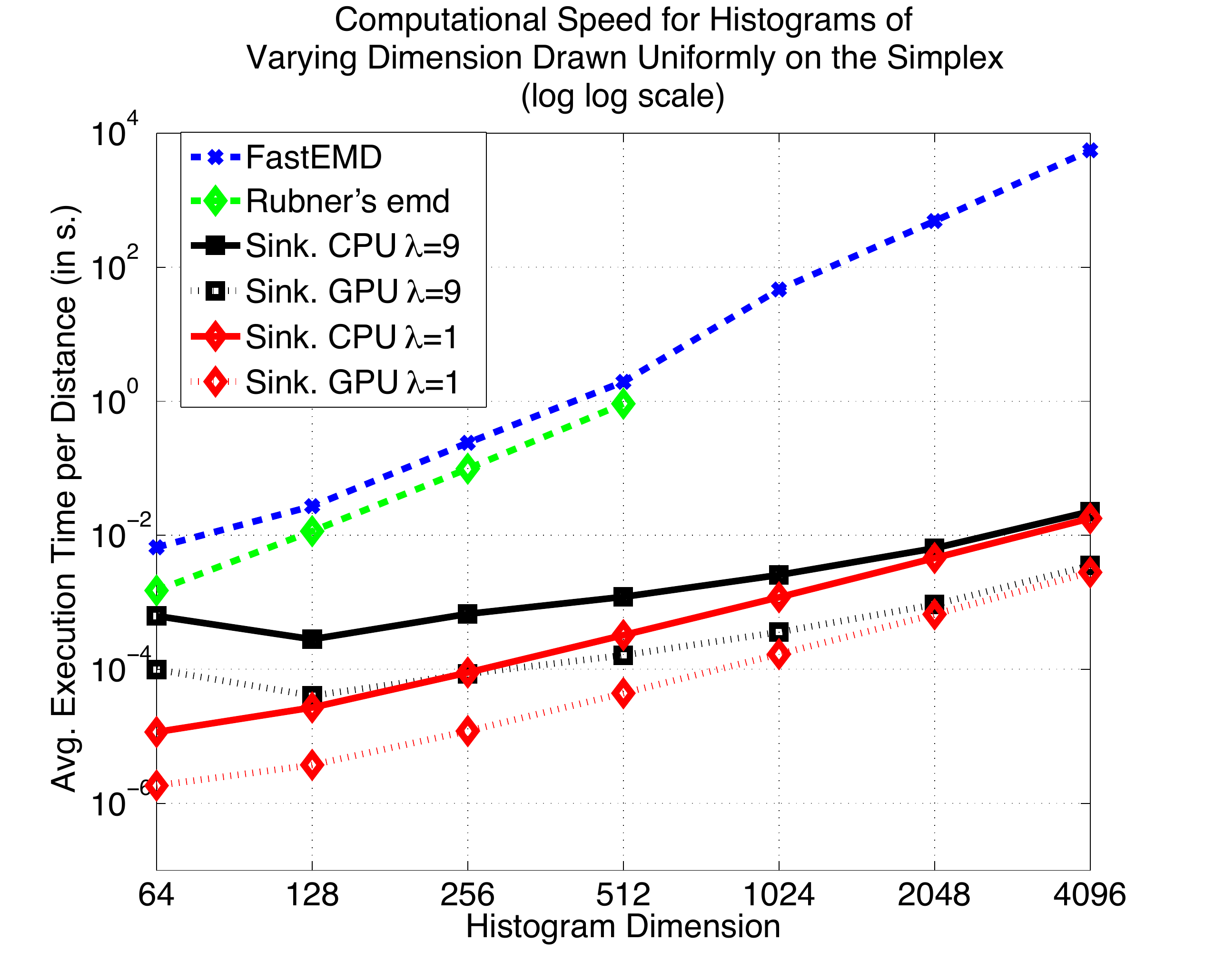}}
\caption{Average computational time required to compute a distance between two histograms sampled uniformly in the $d$ dimensional simplex for varying values of $d$. Sinkhorn distances are run both on a single CPU node and on a GPU card, until the variation in $x$ becomes smaller than $\epsilon=0.01$ in Euclidean norm.}\label{fig:comptime}
\end{figure}
\newpage

\subsection{Several Orders of Magnitude Faster}\label{subsec:computational}
We measure in this section the computational speed of classic optimal transportation distances vs. that of Sinkhorn distances using~\citeauthor{rubner1997earth}'s ~\citeyearpar{rubner1997earth}\footnote{\texttt{http://robotics.stanford.edu/~rubner/emd/default.htm}} and~\citeauthor{Pele-iccv2009}'s~\citeyearpar{Pele-iccv2009}\footnote{\texttt{http://www.cs.huji.ac.il/~ofirpele/FastEMD/code/}, we use \texttt{emd\_hat\_gd\_metric} in these experiments} publicly available implementations. 
We generate points uniformly in the $d$-simplex~\citep{smith2004sampling} and generate random distance matrices $M$ by selecting $d$ points distributed with a spherical Gaussian in dimension $d/10$ to obtain enough variability in the distance matrix. $M$ is then divided by the median of its values, $\texttt{M=M/median(M(:))}$. Sinkhorn distances are implemented in matlab code (see Algorithm~\ref{alg:sk}) while 
\texttt{emd\_mex}, \texttt{emd\_hat\_gd\_metric} are mex/C files. 
The emd distances and Sinkhorn CPU are run on a matlab session with a single working core (2.66 Ghz Xeon). Sinkhorn GPU is run on an NVidia Quadro K5000 card. Following the experimental findings of Section~\ref{subsec:mnist}, we consider two parameters for $\lambda$, $\lambda=1$ and $\lambda=9$. $\lambda=1$ results in a relatively dense matrix $K=e^{-\lambda M}$, with results comparable to that of the Independence kernel, while $\lambda=9$ results in a matrix $K=e^{-\lambda M}$ with mostly negligible values and therefore a matrix with low entropy that is closer to the optimal transportation solution.~\citeauthor{rubner1997earth}'s implementation cannot be run for histograms larger than $d=512$. For large dimensions and on the same CPU, Sinkhorn distances are more than 100.000 faster than EMD solvers given a threshold of $0.01$. Using a GPU results in a speed-up of a supplementary order of magnitude.


\begin{figure}\centering
\fbox{\includegraphics[width=9cm]{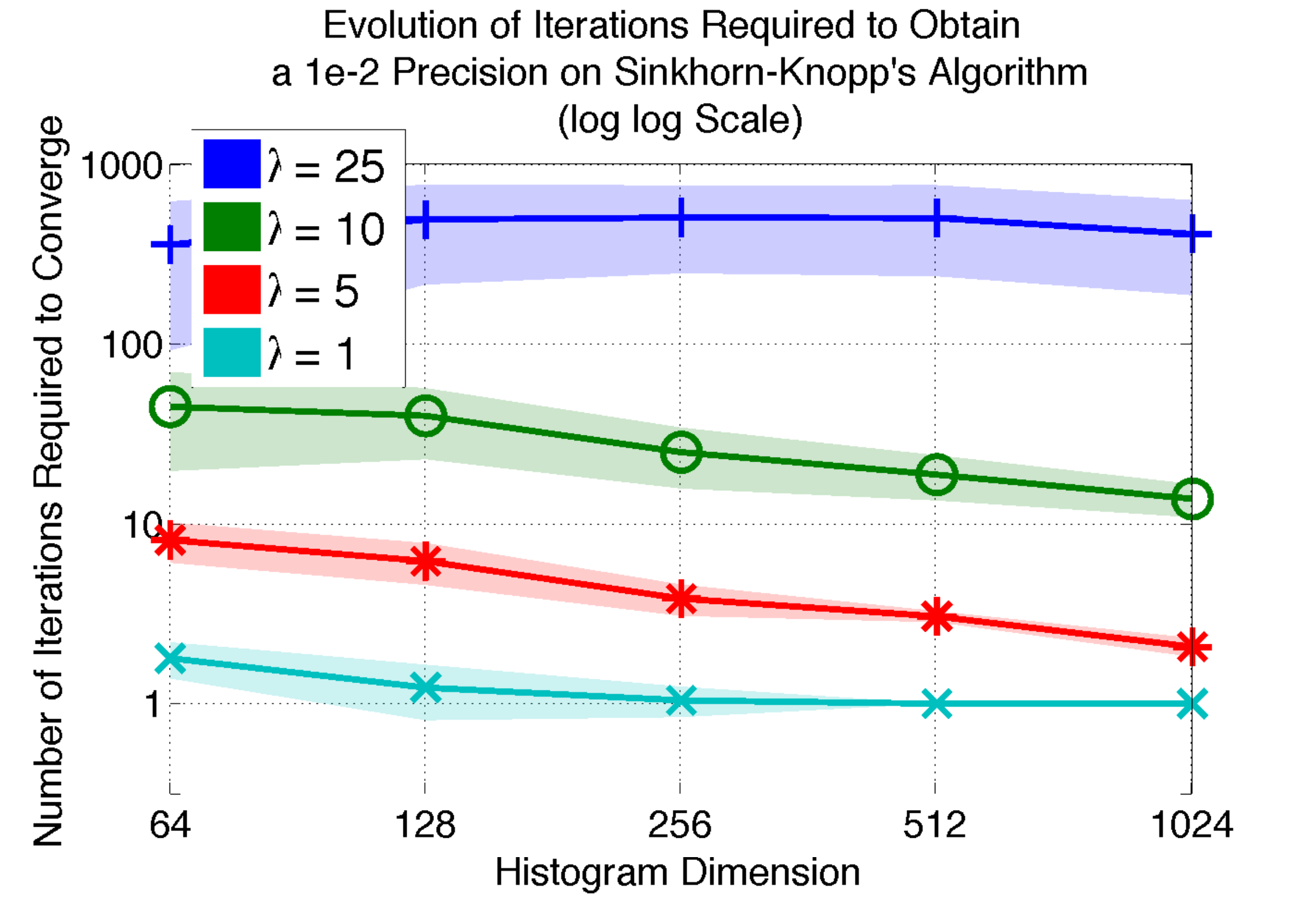}}
\caption{The influence of $\lambda$ on the number of iterations required to converge on histograms uniformly sampled from the simplex.}\label{fig:lambdacompt}
\end{figure}

\subsection{Empirical Complexity}
To provide an accurate picture of the actual number of steps required to guarantee the algorithm's convergence, we replicate the experiments of Section~\ref{subsec:computational} but focus now on the number of iterations of the loop described in Algorithm~\ref{alg:sk}. We use a tolerance of $0.01$ on the norm of the difference of two successive iterations of $x\in\RR^d$. As can be seen in Figure~\ref{fig:lambdacompt}, the number of iterations required so that $\norm{x-x'}_2\leq 0.01$ increases as $e^{-\lambda M}$ becomes diagonally dominant. From a practical perspective, and because keeping track of the change of $x$ at each iteration can be costly on parallel platforms, we recommend setting a fixed number of iterations that only depends on the value of $\lambda$. With that modification, and when computing the distance of a point $r$ to a family of points $C$, we obtain speedups by using GPGPU's which are even larger than those displayed in Figure~\ref{fig:comptime}.

\section{Conclusion}
We have shown that regularizing the optimal transportation problem with an intuitive entropic penalty opens the door for new research directions and potential applications at the intersection of optimal transportation theory and machine learning. This regularization guarantees speed-ups that are effective whatever the structure of the ground metric $M$. Based on preliminary evidence, it seems that Sinkhorn distances do not perform worse than the EMD, and may in fact perform better in applications. Sinkhorn distances are parameterized by a regularization weight $\lambda$ which should be tuned having both computational and performance objectives in mind, but we have not observed a need to establish a trade-off between both. Indeed, reasonably small values of $\lambda$ seem to perform better than large ones.

\section{Appendix: Proofs}

\begin{proof}[Proof of Property~\ref{prop:nsd}]
The set $U_1(r,c)$ contains all joint probabilities $P$ for which $h(P)=h(r)+h(c)$. In that case~\citep[Theorem 2.6.6]{cover91elements} applies and $U_1(r,c)$ can only be equal to the singleton $\{rc^T\}$. If $M$ is negative definite, there exists vectors $(\varphi_1,\cdots,\varphi_d)$ in some Euclidean space $\RR^n$ such that $m_{ij}=\norm{\varphi_i-\varphi_j}_2^2$	through \citep[\S3.3.2]{berg84harmonic}. We thus have that 
$$\begin{aligned}r^TM c&= \sum_{ij}r_i c_j \norm{\varphi_i-\varphi_j}^2= (\sum_{i}r_i \norm{\varphi_i}^2 + \sum_{i} c_i \norm{\varphi_i}^2) - 2\sum_{ij} \dotprod{r_i\varphi_i}{c_j\varphi_j}\\&= r^Tu + c^Tu - 2 r^TK c\end{aligned}$$
where $u_i=\norm{\phi_i}^2$ and $K_{ij}=\dotprod{\varphi_i}{\varphi_j}$. We used the fact that $\sum r_i=\sum c_i=1$ to go from the first to the second equality. $r^TMc$ is thus a n.d. kernel because it is the sum of two n.d. kernels: the first term $(r^Tu+c^Tu)$ is the sum of the same function evaluated separately on $r$ and $c$, and thus a negative definite kernel~\citep[\S3.2.10]{berg84harmonic}; the latter term $-2r^TKu$ is negative definite as minus a positive definite kernel~\citep[Definition \S3.1.1]{berg84harmonic}.\end{proof}	
\emph{Remark.} The proof above suggests a faster way to compute the Independence kernel. Given a matrix $M$, one can indeed pre-compute the vector of norms $u$ as well as a Cholesky factor $L$ of $K$ above to preprocess a dataset of histograms by premultiplying each observations $r_i$ by $L$ and only store $Lr_i$ as well as precomputing its diagonal term $r_i^Tu$. Note that the independence kernel is positive definite on histograms with the same 1-norm, but is no longer positive definite for arbitrary vectors.

\begin{proof}[Proof of Lemma~\ref{lem:gluing}]
	Let $T$ be the a probability distribution on $\{1,\cdots,d\}^d$ whose coefficients are defined as 
\begin{equation}\label{eq:T}t_{ijk}\defeq \frac{p_{ij}q_{jk}}{y_j},\end{equation}
	for all indices $j$ such that $y_j>0$. For indices $j$ such that $y_j=0$, all values $t_{ijk}$ are set to $0$.

Let	$S \defeq [ \sum_{j} t_{ijk}]_{ik}$. $S$ is a transportation matrix between $x$ and $z$. Indeed, 
	$$\begin{aligned}
	\sum_{i} \sum_{j} s_{ijk}&= \sum_{j} \sum_{i}\frac{p_{ij}q_{jk}}{y_j}= \sum_{j} \frac{q_{jk}}{y_j} \sum_{i} p_{ij} =\sum_{j} \frac{q_{jk}}{y_j} y_j=\sum_j q_{jk}=z_k \text{  (column sums) }\\
	\sum_{k} \sum_{j} s_{ijk}&= \sum_{j} \sum_{k}\frac{p_{ij}q_{jk}}{y_j}= \sum_{j} \frac{p_{ij}}{y_j} \sum_{k}q_{jk} =\sum_{j} \frac{p_{ij}}{y_j} y_j=\sum_j p_{ij}=x_i \text{  (row sums) }
	\end{aligned}
	$$
	We now prove that $h(S)\geq h(x)+h(z)-\alpha$. Let $(X,Y,Z)$ be three random variables jointly distributed as $T$. Since by definition of $T$ in Equation~\eqref{eq:T}
	$$p(X,Y,Z)=p(X,Y)p(Y,Z)/p(Y)= p(X)p(Y|X)p(Z|Y),$$ the triplet $(X,Y,Z)$ is a Markov chain $X\rightarrow Y \rightarrow Z$~\citep[Equation 2.118]{cover91elements} and thus, by virtue of the data processing inequality~\citep[Theorem 2.8.1]{cover91elements}, the following inequality between mutual informations applies:
$$\label{eq:dpi}I(X;Y)\geq I(X;Z),\text{ namely } \quad h(X,Z)-h(X)+h(Z)\geq h(X,Y) -h(X)+h(Y)\geq -\alpha.$$\end{proof}

{\small{\bibliographystyle{apa}

\input{fast4.bbl}
}}

\end{document}

%% file: defs.tex
\newcommand{\BEAS}{\begin{eqnarray*}}
\newcommand{\EEAS}{\end{eqnarray*}}
\newcommand{\BEA}{\begin{eqnarray}}
\newcommand{\EEA}{\end{eqnarray}}
\newcommand{\BEQ}{\begin{equation}}
\newcommand{\EEQ}{\end{equation}}
\newcommand{\BIT}{\begin{itemize}}
\newcommand{\EIT}{\end{itemize}}
\newcommand{\BNUM}{\begin{enumerate}}
\newcommand{\ENUM}{\end{enumerate}}

\newcommand{\BA}{\begin{array}}
\newcommand{\EA}{\end{array}}
\newcommand{\BC}{\begin{center}}
\newcommand{\EC}{\end{center}}

\newcommand{\ie}{{\it i.e.}}

\newcommand{\ones}{\mathbf 1}

\newcommand{\diag}{\mathop{\bf diag}}

\newcommand{\argmin}{\mathop{\rm argmin}}

\newcounter{exno}

{\begin{quote}}{\end{quote}}

\makeatother

%% file: marcodefs.tex
\usepackage{dsfont}

\definecolor{puorange}{rgb}{0.70,0.15,0}
\definecolor{bluegray}{rgb}{0.05,0,0.8}
\definecolor{greengray}{rgb}{0.05,0.50,0.15}
\definecolor{darkbrown}{rgb}{0.40,0.2,0.05}
\definecolor{darkcyan}{rgb}{0,0.4,1}
\definecolor{black}{rgb}{0,0,0}

\newcommand{\EBGE}[1]{\boldsymbol{\textcolor{bluegray}{#1}}}

\newcommand{\dotprod}[2]{\ensuremath{\langle #1 , #2\,\rangle}}

\usepackage{enumitem}
\usepackage{enumerate}
\usepackage{amsmath}

\def\RR{\mathbb{R}}

\def\Mcal{\mathcal{M}}

\def\Lcal{\mathcal{L}}

\def\OMIT#1{}

 \DeclareMathOperator{\kl}{kl}

\DeclareMathOperator{\defi}{def}

\DeclareMathOperator{\defeq}{\overset{\defi}{=}}

\usepackage{enumerate,multirow}

\providecommand{\norm}[1]{\lVert#1\rVert}

\usepackage{tabularx}

\makeatletter
\newif\if@borderstar
\def\bordermatrix{\@ifnextchar*{%
  \@borderstartrue\@bordermatrix@i}{\@borderstarfalse\@bordermatrix@i*}%
}
\def\@bordermatrix@i*{\@ifnextchar[{%
  \@bordermatrix@ii}{\@bordermatrix@ii[()]}
}
\def\@bordermatrix@ii[#1]#2{%
  \begingroup
    \m@th\@tempdima8.75\p@\setbox\z@\vbox{%
      \def\cr{\crcr\noalign{\kern 2\p@\global\let\cr\endline }}%
      \ialign {$##$\hfil\kern 2\p@\kern\@tempdima & \thinspace %
      \hfil $##$\hfil && \quad\hfil $##$\hfil\crcr\omit\strut %
      \hfil\crcr\noalign{\kern -\baselineskip}#2\crcr\omit %
      \strut\cr}}%
    \setbox\tw@\vbox{\unvcopy\z@\global\setbox\@ne\lastbox}%
    \setbox\tw@\hbox{\unhbox\@ne\unskip\global\setbox\@ne\lastbox}%
    \setbox\tw@\hbox{%
      $\kern\wd\@ne\kern -\@tempdima\left\@firstoftwo#1%
        \if@borderstar\kern2pt\else\kern -\wd\@ne\fi%
      \global\setbox\@ne\vbox{\box\@ne\if@borderstar\else\kern 2\p@\fi}%
      \vcenter{\if@borderstar\else\kern -\ht\@ne\fi%
        \unvbox\z@\kern-\if@borderstar2\fi\baselineskip}%
        \if@borderstar\kern-2\@tempdima\kern2\p@\else\,\fi\right\@secondoftwo#1 $%
    }\null \;\vbox{\kern\ht\@ne\box\tw@}%
  \endgroup
}
\makeatother

%% file: newtransp+4.pdf_tex

\begingroup
  \makeatletter
  \providecommand\color[2][]{%
    \errmessage{(Inkscape) Color is used for the text in Inkscape, but the package 'color.sty' is not loaded}
    \renewcommand\color[2][]{}%
  }
  \providecommand\transparent[1]{%
    \errmessage{(Inkscape) Transparency is used (non-zero) for the text in Inkscape, but the package 'transparent.sty' is not loaded}
    \renewcommand\transparent[1]{}%
  }
  \providecommand\rotatebox[2]{#2}
  \ifx\svgwidth\undefined
    \setlength{\unitlength}{443.26467577pt}
  \else
    \setlength{\unitlength}{\svgwidth}
  \fi
  \global\let\svgwidth\undefined
  \makeatother
  \begin{picture}(1,0.52865982)%
    \put(0,0){\includegraphics[width=\unitlength]{newtransp+4.pdf}}%
    \put(0.15220005,0.44273643){\color[rgb]{0,0,0}\makebox(0,0)[lb]{\smash{$\EBGE{M}$}}}%
    \put(0.07258801,0.00851283){\color[rgb]{0,0,0}\makebox(0,0)[lb]{\smash{$d_{M,\alpha}(r,c)=\langle P^\star,M\rangle$}}}%
    \put(0.60534081,0.16431994){\color[rgb]{0,0,0}\makebox(0,0)[lb]{\smash{$U(r,c)$}}}%
    \put(0.39353147,0.27599693){\color[rgb]{0,0,0}\makebox(0,0)[lb]{\smash{$rc^T$}}}%
    \put(0.20726811,0.15940298){\color[rgb]{0,0,0}\makebox(0,0)[lb]{\smash{$P^\star$}}}%
    \put(0.28095043,0.35013798){\color[rgb]{1,0,0}\makebox(0,0)[lb]{\smash{$\ERED{U_\alpha(r,c)=\{P\in U(r,c) | \,\KL(P\|rc^T) \leq \alpha\}}$}}}%
  \end{picture}%
\endgroup

%% file: fast4.bbl
\begin{thebibliography}{}

\bibitem[\protect\astroncite{Ahuja et~al.}{1993}]{ahuja1993network}
Ahuja, R., Magnanti, T., and Orlin, J. (1993).
\newblock {\em Network Flows: Theory, Algorithms and Applications}.
\newblock Prentice Hall.

\bibitem[\protect\astroncite{Andoni et~al.}{2009}]{indyk2009}
Andoni, A., Ba, K.~D., Indyk, P., and Woodruff, D. (2009).
\newblock Efficient sketches for earth-mover distance, with applications.
\newblock In {\em Foundations of Computer Science (FOCS) 2009.}, pages 324
  --330.

\bibitem[\protect\astroncite{Avis}{1980}]{avis1980extreme}
Avis, D. (1980).
\newblock On the extreme rays of the metric cone.
\newblock {\em Canadian Journal of Mathematics}, 32(1):126--144.

\bibitem[\protect\astroncite{Berg et~al.}{1984}]{berg84harmonic}
Berg, C., Christensen, J., and Ressel, P. (1984).
\newblock {\em Harmonic Analysis on Semigroups}.
\newblock Number 100 in Graduate Texts in Mathematics. Springer Verlag.

\bibitem[\protect\astroncite{Brickell et~al.}{2008}]{brickell2008metric}
Brickell, J., Dhillon, I., Sra, S., and Tropp, J. (2008).
\newblock The metric nearness problem.
\newblock {\em SIAM J. Matrix Anal. Appl}, 30(1):375--396.

\bibitem[\protect\astroncite{Brualdi}{2006}]{brualdi2006combinatorial}
Brualdi, R.~A. (2006).
\newblock {\em Combinatorial matrix classes}, volume 108.
\newblock Cambridge University Press.

\bibitem[\protect\astroncite{Cover and Thomas}{1991}]{cover91elements}
Cover, T. and Thomas, J. (1991).
\newblock {\em Elements of Information Theory}.
\newblock Wiley \& Sons.

\bibitem[\protect\astroncite{Dattorro}{2005}]{dattorro2005convex}
Dattorro, J. (2005).
\newblock {\em Convex optimization \& Euclidean distance geometry}.
\newblock Meboo Publishing USA.

\bibitem[\protect\astroncite{Dud{\'\i}k and Schapire}{2006}]{dudik2006maximum}
Dud{\'\i}k, M. and Schapire, R.~E. (2006).
\newblock Maximum entropy distribution estimation with generalized
  regularization.
\newblock In {\em Learning Theory}, pages 123--138. Springer.

\bibitem[\protect\astroncite{Erlander and Stewart}{1990}]{erlander1990gravity}
Erlander, S. and Stewart, N. (1990).
\newblock {\em {The gravity model in transportation analysis: theory and
  extensions}}.
\newblock Vsp.

\bibitem[\protect\astroncite{Ferradans et~al.}{2013}]{ferradans2013regularized}
Ferradans, S., Papadakis, N., Rabin, J., Peyr{\'e}, G., Aujol, J.-F., et~al.
  (2013).
\newblock Regularized discrete optimal transport.
\newblock In {\em International Conference on Scale Space and Variational
  Methods in Computer Vision}, pages 1--12.

\bibitem[\protect\astroncite{Good}{1963}]{good1963maximum}
Good, I. (1963).
\newblock Maximum entropy for hypothesis formulation, especially for
  multidimensional contingency tables.
\newblock {\em The Annals of Mathematical Statistics}, pages 911--934.

\bibitem[\protect\astroncite{Grauman and Darrell}{2004}]{GraumanD04}
Grauman, K. and Darrell, T. (2004).
\newblock Fast contour matching using approximate earth mover's distance.
\newblock In {\em IEEE Conf. Vision and Patt. Recog.}, pages 220--227.

\bibitem[\protect\astroncite{Gudmundsson et~al.}{2007}]{gudmundsson2007small}
Gudmundsson, J., Klein, O., Knauer, C., and Smid, M. (2007).
\newblock Small manhattan networks and algorithmic applications for the earth
  mover’s distance.
\newblock In {\em Proceedings of the 23rd European Workshop on Computational
  Geometry}, pages 174--177.

\bibitem[\protect\astroncite{Indyk and Thaper}{2003}]{indyk}
Indyk, P. and Thaper, N. (2003).
\newblock Fast image retrieval via embeddings.
\newblock In {\em 3rd International Workshop on Statistical and Computational
  Theories of Vision (at ICCV)}.

\bibitem[\protect\astroncite{Jaynes}{1957}]{PhysRev.106.620}
Jaynes, E.~T. (1957).
\newblock Information theory and statistical mechanics.
\newblock {\em Phys. Rev.}, 106:620--630.

\bibitem[\protect\astroncite{Knight}{2008}]{knight2008sinkhorn}
Knight, P.~A. (2008).
\newblock The sinkhorn-knopp algorithm: convergence and applications.
\newblock {\em SIAM Journal on Matrix Analysis and Applications},
  30(1):261--275.

\bibitem[\protect\astroncite{Levina and Bickel}{2001}]{levina2001earth}
Levina, E. and Bickel, P. (2001).
\newblock The earth mover's distance is the mallows distance: some insights
  from statistics.
\newblock In {\em Proceedings of the Eighth IEEE International Conference on
  Computer Vision}, volume~2, pages 251--256. IEEE.

\bibitem[\protect\astroncite{Ling and Okada}{2007}]{ling2007efficient}
Ling, H. and Okada, K. (2007).
\newblock An efficient earth mover's distance algorithm for robust histogram
  comparison.
\newblock {\em IEEE transactions on Patt. An. and Mach. Intell.}, pages
  840--853.

\bibitem[\protect\astroncite{Naor and Schechtman}{2007}]{naor-2005}
Naor, A. and Schechtman, G. (2007).
\newblock Planar earthmover is not in l$_{\mbox{1}}$.
\newblock {\em SIAM J. Comput.}, 37(3):804--826.

\bibitem[\protect\astroncite{Orlin}{1993}]{orlin1993faster}
Orlin, J.~B. (1993).
\newblock A faster strongly polynomial minimum cost flow algorithm.
\newblock {\em Operations research}, 41(2):338--350.

\bibitem[\protect\astroncite{Pele and Werman}{2009}]{Pele-iccv2009}
Pele, O. and Werman, M. (2009).
\newblock Fast and robust earth mover's distances.
\newblock In {\em ICCV'09}.

\bibitem[\protect\astroncite{Rubner et~al.}{1997}]{rubner1997earth}
Rubner, Y., Guibas, L., and Tomasi, C. (1997).
\newblock The earth mover’s distance, multi-dimensional scaling, and
  color-based image retrieval.
\newblock In {\em Proceedings of the ARPA Image Understanding Workshop}, pages
  661--668.

\bibitem[\protect\astroncite{Shirdhonkar and
  Jacobs}{2008}]{shirdhonkar2008approximate}
Shirdhonkar, S. and Jacobs, D. (2008).
\newblock Approximate earth mover’s distance in linear time.
\newblock In {\em CVPR 2008}, pages 1--8. IEEE.

\bibitem[\protect\astroncite{Sinkhorn and Knopp}{1967}]{sinkhorn1967concerning}
Sinkhorn, R. and Knopp, P. (1967).
\newblock Concerning nonnegative matrices and doubly stochastic matrices.
\newblock {\em Pacific J. Math}, 21(2):343--348.

\bibitem[\protect\astroncite{Smith and Tromble}{2004}]{smith2004sampling}
Smith, N.~A. and Tromble, R.~W. (2004).
\newblock Sampling uniformly from the unit simplex.
\newblock {\em Johns Hopkins University, Tech. Rep}, 10:15--20.

\bibitem[\protect\astroncite{Villani}{2003}]{villani}
Villani, C. (2003).
\newblock {\em Topics in Optimal Transportation}, volume~58.
\newblock AMS Graduate Studies in Mathematics.

\bibitem[\protect\astroncite{Villani}{2009}]{villani09}
Villani, C. (2009).
\newblock {\em Optimal transport: old and new}, volume 338.
\newblock Springer Verlag.

\end{thebibliography}
